\documentclass{article}
\usepackage{spconf}
\usepackage{url}
\usepackage{amsfonts}
\usepackage{color}
\usepackage{amssymb}
\usepackage{bm}
\usepackage{amsmath,graphicx}
\usepackage{tabularx}
\usepackage[ruled,norelsize]{algorithm2e}

\usepackage{amsthm}

\newtheorem{theorem}{Theorem}

\newtheorem{problem}{Problem}

\makeatletter
\newcommand{\removelatexerror}{\let\@latex@error\@gobble}
\makeatother

\title{Common Mode Patterns for Supervised Tensor Subspace Learning}
\name{Konstantinos Makantasis$^1$, Anastasios Doulamis$^2$, Nikolaos Doulamis$^2$, and Athanasios Voulodimos$^3$}
\address{$^1$ School of Production Engineering and Management, Technical University of Crete, Chania, Greece \\
$^2$ School of Rural and Surveying Engineering, National Technical University of Athens, Athens, Greece \\
$^3$ Department of Informatics and Computer Engineering, University of West Attica, Athens, Greece}
\begin{document}
%

\maketitle
\begin{abstract}
In this work we propose a method for reducing the dimensionality of tensor objects in a binary classification framework. The proposed Common Mode Patterns method takes into consideration the labels' information, and ensures that tensor objects that belong to different classes do not share common features after the reduction of their dimensionality. We experimentally validate the proposed supervised subspace learning technique and compared it against Multilinear Principal Component Analysis using a publicly available hyperspectral imaging dataset. Experimental results indicate that the proposed CMP method can efficiently reduce the dimensionality of tensor objects, while, at the same time, increasing the inter-class separability. 
\end{abstract}
%

\section{Introduction}
\label{sec:intro}

Advances in sensing technologies have led to the continuous generation of massive multidimensional data, used in a wide range of applications. Their successful exploitation, however, is directly linked to the effectiveness of pattern recognition methods employed for their analysis. Despite the high dimensionality, this kind of data is often characterized by large amounts of redundancy, occupying a subspace of the input space \cite{shakhnarovich2011face}. In this context, feature extraction for subspace learning plays a crucial role towards the mapping of high-dimensional data to a low-dimensional space \cite{lu2008mpca, nie2009extracting, lai2013sparse, hu2011incremental}. However, feature extraction is often a challenging task due to the complex distribution of input data \cite{lu2011survey}, especially in cases of limited training samples  \cite{makantasis2017tensor, makantasis2018tensor}. 

The goal of feature extraction is to extract information regarding the underlying nature of the data. Unsupervised feature extraction methods in particular aim at capturing the principal statistical relation within the data and represent it in lower dimension spaces. This type of methods is referred to as unsupervised subspace learning and includes techniques such as 2D Principal Component Analysis (2D-PCA) \cite{yang2004two}, Generalized Low Rank Approximation of Matrices (GLRAM) \cite{ye2005generalized}, Concurrent Subspace Analysis \cite{xu2008reconstruction} and Multilinear Principal Component Analysis (MPCA) \cite{lu2008mpca}. Such methods can also decrease the computational cost of pattern recognition algorithms (e.g. for classification or regression) through the reduction of data dimensionality.  

The main objective of pattern recognition, however, is the extraction of features capable of discriminating different classes. Although unsupervised subspace learning can provide a valuable tool for data analysis, the features extracted are not necessarily those salient features required to discriminate among pattern classes, since the problem of finding discriminative features is conceptually and fundamentally different than mapping data to a lower dimension space. Different sets of features should be used for different classes, which means that the feature extraction process should be conducted in a supervised manner (supervised subspace learning).  

In this paper, we propose a supervised subspace learning method, which is motivated by the Common Spatial Patterns (CSP) \cite{ramoser2000optimal, blankertz2008optimizing} algorithm. The CSP algorithm is based on a modification of the Karhunen-Loeve expansion \cite{fukunaga1970application}, aiming at extracting features that increase inter-class separability. The application of CSP, however, is restricted to 2D data. Motivated by this fact, we extend the CSP algorithm to tensor objects of arbitrary order. In particular, we extract the common patterns corresponding to each mode of tensor objects - hence naming the former Common Mode Patterns (CMP) - that increase the separability between two classes.

\section{Preliminaries}
\label{sec:preliminaries}
In this section we present some tensor algebra definitions and operations that will be used throughout this work. Tensor objects are denoted in calligraphic uppercase letters, matrices in bold uppercase letters, vectors in bold lowercase letters and scalars in lowercase letters.  


\vspace{0.02in}
\noindent \textbf{Tensor matricization}. Mode-$n$ matricization maps a tensor $\mathcal B$ into a $I_n \times \prod_{n' \neq n}I_{n'}$ matrix  $\bm B_{(n)}$, by arranging the mode-$n$ fibers to be the columns of the resulting matrix. 

\vspace{0.02in}
\noindent \textbf{$\bm n$-mode product}. The $n$-mode product of a tensor $\mathcal A \in \mathbb R^{I_1 \times \cdots \times I_N}$ and a matrix $\bm B \in \mathbb R^{J \times I_n}$ denoted as $\mathcal A \times_n \bm B$ is a tensor in $\mathbb R^{I_1 \times I_2 \times \cdots \times I_{n-1} \times J \times I_{n+1} \times \cdots \times I_N}$ with entries
	\begin{equation}
	\begin{split}
	(\mathcal A \times_n \bm B)(i_1, \cdots, i_{n-1}, j, &i_{n+1},\cdots, i_N) = \\
	&\sum_{i_n} \mathcal A(i_1,\cdots,i_N) \bm B(j, i_n). \nonumber
	\end{split}
	\end{equation}
	
\noindent \textbf{Scalar product}. The scalar product of two tensors $\mathcal A, \mathcal B \in \mathbb R^{I_1 \times \cdots \times I_N}$ is denoted as $\langle \mathcal A, \mathcal B \rangle$ and is equal to $\langle \textit{vec}(\mathcal A), \textit{vec}(\mathcal B) \rangle$.

\vspace{0.02in}
\noindent \textbf{Tensor norm}. The Frobenius norm of a tensor $\mathcal A$ is defined as $||\mathcal A||_F = \sqrt{\langle\mathcal A, \mathcal A\rangle}$.

\vspace{0.02in}
\noindent \textbf{Average total scatter}. The average total scatter of a set of tensors $\{\mathcal A_m\}_{m=1}^M$ is defined as 
\begin{equation}
\Psi_{\mathcal A} = \frac{1}{M}\sum_{m=1}^M ||\mathcal A_m - \bar{\mathcal A}||_F^2,
\end{equation}
where $\bar{\mathcal A} = \frac{1}{M} \sum_{m=1}^M \mathcal A_m$.

\vspace{0.02in}
\noindent \textbf{Average mode-$\bm n$ scatter matrix}. The average mode-$n$ scatter matrix of a set of tensors $\{\mathcal A_m\}_{m=1}^M$ is defined as 
\begin{equation}
\Psi_{n, \mathcal A} = \frac{1}{M}\sum_{m=1}^M ||\bm A_{(n),m} - \bar{\bm A}_{(n)}||_F^2, 
\end{equation}
where $\bar{\bm A}_{(n)} = \frac{1}{M} \sum_{m=1}^M \bm A_{(n),m}$, and $\bm A_{(n),m}$ is the $n$-mode matricization of $\mathcal A_m$.

\section{Problem Formulation}
We consider a binary classification problem, where the samples are tensor objects. Let $\{\mathcal A_m^{(i)} \in\mathbb R^{I_1 \times I_2 \times \cdots \times I_N} \}_{m=1}^{M_i}$, $i=1,2$ be a set of $M_i$ samples that belong to the $i$-th class, and $\{\bm U_n\}_{n=1}^N$ a set of matrices, where $\bm U_n \in \mathbb R^{I_n \times P_n}$ with $P_n \leq I_n$. The projection of any $\mathcal A_m^{(i)}$ onto the subspace $\mathbb R^{P_1, \cdots, P_N}$ is defined as
\begin{equation}
\label{eq:tensor_projection}
\mathcal S_m^{(i)} = \mathcal A_m^{(i)} \times_1 \bm U_1^T \times_2 \bm U_2^T \cdots \times_N \bm U_N^T.
\end{equation}
A matrix representation of this projection can be obtained through the mode-$n$ matricization of $\mathcal S_m^{(i)}$ and $\mathcal A_m^{(i)}$ as 
\begin{equation}
\label{eq:projection}
\bm S_{(n),m}^{(i)} = \bm U_n^T \cdot \bm A_{(n),m}^{(i)} \cdot \bm U_{\bm \Phi_n}
\end{equation}
with
\begin{equation}
\label{eq:Phi}
\begin{split}
\bm U_{\bm \Phi_n} = \bm U_{n+1} &\otimes \bm U_{n+2} \otimes \cdots \\
&\otimes \bm U_{N} \otimes \bm U_{1} \otimes \bm U_{2} \otimes \cdots \otimes \bm U_{n-1}.
\end{split}
\end{equation}
In relation (\ref{eq:Phi}), the operator $\otimes$ denotes the Kronecker product.

The objective of this work is to project the tensor samples $\mathcal A_m^{(i)}$ onto a subspace, where the explained variance for each class is maximized, and the different properties of each class are emphasized. The projected samples thus minimize information loss and can discriminate between the two different pattern classes. By assuming that the variance of a class can be measured by the average total scatter of the tensor samples belonging to this class \cite{lu2008mpca}, we can formally define the problem that needs to be solved for achieving the aforementioned objective.

\begin{problem}
	Estimate a single set $\{\tilde{\bm U}_n\}_{n=1}^N$ of projection matrices [see (\ref{eq:tensor_projection})] that satisfy
	\begin{equation}
	\label{eq:objective}
	\{\tilde{\bm U}_n, \: n=1,\cdots, N\} = \arg \max_{\bm U_1, \cdots, \bm U_N} \Psi_{\mathcal S^{(i)}}
	\end{equation}
	for $i=1,2$, such that the projected samples that belong to different classes will not share common important features.
\end{problem}

In Problem 1, defined above, $\mathcal S_m^{(i)}$ is the projection of $\mathcal A_m^{(i)}$ onto a subspace using the projection matrices $\{\tilde{\bm U}_n\}$, while
\begin{equation}
\label{eq:Psi_i}
\Psi_{\mathcal S^{(i)}} = \frac{1}{M_i}\sum_{m=1}^{M_i} ||\mathcal S_m^{(i)} - \bar{\mathcal S}^{(i)}||_F^2
\end{equation}
and 
\begin{equation}
\bar{\mathcal S}^{(i)} = \frac{1}{M_i} \sum_{m=1}^{M_i} \mathcal S_m^{(i)}.
\end{equation}

\noindent \textit{Remark 1.} Suppose that $\{\tilde{\bm U}_n, \: n=1,\cdots, N\}$ is a set of projection matrices that satisfies (\ref{eq:objective}) either for $i=1$ or for $i=2$. Then, as is shown in \cite{lu2008mpca}, each matrix $\bm U_n$, $n=1,\cdots,N$, consists of the $P_n$ eigenvectors corresponding to the largest $P_n$ eigenvalues of the matrix
\begin{equation}
\label{eq:Phi_n}
\begin{split}
\bm \Phi_n^{(i)} = \frac{1}{M_i}\sum_{m=1}^{M_i} (\bm A_{(n),m}^{(i)} - &\bar{\bm A}_{(n)} ) \cdot \tilde{\bm U}_{\bm \Phi_n} \\
& \tilde{\bm U}_{\bm \Phi_n}^T \cdot (\bm A_{(n),m}^{(i)} - \bar{\bm A}_{(n)} )^T,
\end{split}
\end{equation}
where $ \tilde{\bm U}_{\bm \Phi_n}$ is as in (\ref{eq:Phi}). However, in Problem 1, the set of projection matrices should satisfy (\ref{eq:objective}) both for $i=1$ and $i=2$, and, at the same time, the resulting projection should emphasize different sets of features for each class. 

\section{Common Mode Patterns}
\subsection{Normalization Process}
For the CMP algorithm to extract those important features that are required for separating two pattern classes, a preprocessing step, in the sense of a normalization process, is necessary. We hereby present this normalization process.

Suppose that we have at our disposal a set $\{\mathcal B_m^{(i)}, m=1,\cdots,M_i, \: i=1,2\}$ of $M_i$ raw tensor measurements (samples) in $\mathbb R^{I_1 \times \cdots \times I_N}$ that belong to the $i$-th class. Based on these samples we can define the matrix
\begin{equation}
\label{eq:R_n}
\begin{split}
\bm R_n^{(i)}= \frac{1}{M_i}\sum_{m=1}^{M_i} \Big( \bm B_{(n),m}^{(i)} &- \bar{\bm B}_{(n)}^{(i)} \Big)\cdot \\ 
&\Big( \bm B_{(n),m}^{(i)} - \bar{\bm B}_{(n)}^{(i)} \Big)^T. 
\end{split}
\end{equation}
For every $m=1\cdots M_i$ and $i=1,2$ the matrix
\begin{equation}
\Big( \bm B_{(n),m}^{(i)} - \bar{\bm B}_{(n)}^{(i)} \Big)\cdot 
\Big( \bm B_{(n),m}^{(i)} - \bar{\bm B}_{(n)}^{(i)} \Big)^T \nonumber
\end{equation}
is symmetric. Hence, matrix $\bm R_n^{(i)}$ is also symmetric, since it is the weighted sum of symmetric matrices. 

Let us define the symmetric matrix $\bm R_n = \bm R_n^{(1)} + \bm R_n^{(2)}$. Since $\bm R_n$ is symmetric, there exists the transformation matrix
\begin{equation}
\label{eq:whitening}
\bm Z_n = \text{diag}(\bm \lambda_n)^{-1/2} \cdot \bm V_n^T
\end{equation}
such that 
\begin{eqnarray}
\label{eq:12}
\bm Z_n \cdot \bm R_n \cdot \bm Z_n^T = \bm Z_n \cdot \bm R_n^{(1)} \cdot \bm Z_n^T + \bm Z_n \cdot \bm R_n^{(2)} \cdot \bm Z_n^T = \bm I.
\end{eqnarray}
In (\ref{eq:whitening}), $\text{diag}(\bm \lambda_n)$ stands for the diagonal matrix of eigenvalues of $\bm R_n$, while $\bm V_n$ for the matrix of eigenvectors of $\bm R_n$.

Following the above normalization process, we define the mode-$n$ matricization of tensor objects $\{\mathcal A_m^{(i)}\}_{m=1}^{M_i}$, $i=1,2$ that need to be projected onto a subspace as
\begin{equation}
\label{eq:normalization}
\bm A_{(n),m}^{(i)} = \bm Z_n \bm B_{(n),m}^{(i)}.
\end{equation}
The normalization process takes place before the projection, and actually corresponds to a linear transformation, which is applied on the tensor objects.

\subsection{The CMP Algorithm}
This section presents the CMP algorithm, which constitutes the core contribution of this paper. The CMP algorithm is based on Theorem 1 below.
\begin{theorem}
	Let $\{\tilde{\bm U}_n, \: n=1,\cdots, N\} $ be the solution to Problem 1. Then, given all other projection matrices $\tilde{\bm U}_1,\cdots,\tilde{\bm U}_{n-1}, \tilde{\bm U}_{n+1},\cdots,\tilde{\bm U}_N$, matrix $\tilde{\bm U}_n$ consists of the $P_n/2$ eigenvectors corresponding to the largest eigenvalues of the matrix $\bm \Phi_n^{(1)}$ and the $P_n/2$ eigenvectors corresponding to the largest eigenvalues of the matrix $\bm \Phi_n^{(2)}$.
\end{theorem}

\begin{proof}
	From the definition of Frobenius norm for a tensor and that for a matrix, $||\mathcal A||_F = ||\bm A_{(n)}||_F$, and from  Eq. (\ref{eq:projection}), it holds that
	\begin{equation}
	\begin{split}
	\Psi_{S^{(i)}} &=  \frac{1}{M_i}\sum_{m=1}^{M_i} ||\mathcal S_m^{(i)} - \bar{\mathcal S}^{(i)}||_F^2 \\
	& = \frac{1}{M_i}\sum_{m=1}^{M_i} || \tilde{\bm U}_n^T \cdot \big(\bm A_{(n),m}^{(i)} - \bar{\bm A}_{(n)}^{(i)}\big) \cdot \tilde{\bm U}_{\bm \Phi_n}||_F^2
	\end{split}
	\end{equation}
	Moreover, from Eq. (\ref{eq:Phi_n}) 
	$\Psi_{S^{(i)}}$ can be written as
	\begin{equation}
	\label{eq:15}
	\begin{split}
	\Psi_{S^{(i)}} &= \frac{1}{M_i}\sum_{m=1}^{M_i} \text{trace} \Big( \tilde{\bm U}_n^T \cdot \big(\bm A_{(n),m}^{(i)} - \bar{\bm A}_{(n)}^{(i)}\big) \cdot \tilde{\bm U}_{\bm \Phi_n} \cdot \\
	&\;\;\;\;\;\;\;\;\;\;\;\;\;\;\;\;\;\;\;\;\;\;\;\;\;\; \tilde{\bm U}_{\bm \Phi_n}^T \cdot \big(\bm A_{(n),m}^{(i)} - \bar{\bm A}_{(n)}^{(i)}\big)^T \cdot \tilde{\bm U}_n \Big) \\
	& =  \frac{1}{M_i} \text{trace} \big(\tilde{\bm U}_n^T \cdot \bm \Phi_n^{(i)} \tilde{\bm U}_n \big)
	\end{split}.
	\end{equation}
	The maximum trace of $(\tilde{\bm U}_n^T \cdot \bm \Phi_n^{(i)} \tilde{\bm U}_n)$ is obtained if $\tilde{\bm U}_n$ consists of the $P_n$ eigenvectors of matrix $\bm \Phi_n^{(i)}$ corresponding to the largest $P_n$ eigenvalues. Since we want to maximize $\Psi_{S^{(i)}}$ simultaneously for $i=1$ and $i=2$, matrix $\tilde{\bm U}_n$ will consist of the $P_n/2$ eigenvectors corresponding to the largest eigenvalues of matrix $\bm \Phi_n^{(1)}$ and the $P_n/2$ eigenvectors corresponding to the largest eigenvalues of matrix $\bm \Phi_n^{(2)}$.
	
	Let us denote as $\bm C^{(i)}$ the matrix
	\begin{equation}
	\label{eq:16}
	\bm C^{(i)}_{(n)} =  \frac{1}{M_i}\sum_{m=1}^{M_i} \Big( \big(\bm A_{(n),m}^{(i)} - \bar{\bm A}_{(n)}^{(i)}  \big) \big(\bm A_{(n),m}^{(i)} - \bar{\bm A}_{(n)}^{(i)}\big)  \Big).
	\end{equation}
	After the normalization process [see Eq. (\ref{eq:12})]
	\begin{equation}
	\label{eq:17}
	\bm C^{(1)}_{(n)} + \bm C^{(2)}_{(n)} =\bm I.
	\end{equation}
	The eigenvalues and eigenvectors of $\bm C^{(i)}_{(n)}$ are given by
	\begin{equation}
	\label{eq:19}
	\bm C^{(i)}_{(n)} \bm V_{(n)}^{(i)} =  \Big(\bm I - \bm C^{(j)}_{(n)} \Big) \bm V_{(n)}^{(i)}   = \bm \lambda_{(n)}^{(i)}  \bm V_{(n)}^{(i)},
	\end{equation}
	with $i \neq j$. From Eq. (\ref{eq:17}), and (\ref{eq:19}) we have that 
	\begin{equation}
	\label{eq:20}
	\bm \lambda_{(n)}^{(2)} = \big(\bm I - \bm \lambda_{(n)}^{(1)} \big).
	\end{equation}
	The same holds for matrices $\frac{1}{M_i}\sum_{m=1}^{M_i} ((\bm A_{(n),m}^{(i)} - \bar{\bm A}_{(n)}^{(i)} ) \cdot \tilde{\bm U}_{\bm \Phi_n}  \tilde{\bm U}_{\bm \Phi_n}^T (\bm A_{(n),m}^{(i)} - \bar{\bm A}_{(n)}^{(i)}) ),$ in relation (\ref{eq:15}),
	since they are similar (i.e., have the same eigenvalues) with the matrices in relation (\ref{eq:16}). For this to become clearer, note that $ \tilde{\bm U}_{\bm \Phi_n}^T =  \tilde{\bm U}_{\bm \Phi_n}^{-1}$, since $\tilde{\bm U}_{\bm \Phi_n}$ is the Kronecker product of orthogonal matrices, and thus it is also orthogonal.
	
From Eq. (\ref{eq:20}), we have that the important features for the first class are the least important features for the second class, and vice versa. This means that after the projection, the two classes cannot share common important features.
\end{proof}
The CMP algorithm is presented in Algorithm \ref{alg:1}. Please note that during the estimation of the set $\{\bm U_n\}$ only the $\bm \Phi_n^{(1)}$  matrices are used. $\bm \Phi_n^{(2)}$ matrices are not employed in the algorithm since matrices $\bm \Phi_n^{(1)}$  and $\bm \Phi_n^{(2)}$ have the same eigenvectors and reversely ordered eigenvalues.

\begin{figure}[!h]
	{\begingroup
		\removelatexerror
		\begin{algorithm}[H]
			\caption{Estimation of matrices $\{\bm U_n\}_{n=1}^N$}
			\label{alg:1}
			\SetAlgoLined
			1. Set $\bm U_n=\bm I_n$, for $n=1, \cdots, N$\\
			2. Calculate $\bm R_n$ and $\bm Z_n$ using relations (\ref{eq:R_n}) and (\ref{eq:whitening}) for $n=1,\cdots,N$\\
			3. Normalize tensor samples using relation (\ref{eq:normalization})  \\
			4. \Repeat{termination criteria are met}{
				\For{$n=1,...,N$} {
					4.1 Calculate the matrix $\bm U_{\bm \Phi_n}$ of relation.(\ref{eq:Phi})\\
					4.2 Calculate the matrix $\bm \Phi_n^{(1)}$ of relation.(\ref{eq:Phi_n})\\
					4.3 Calculate the eigenvectors of $\bm \Phi_n^{(1)}$ \\
					4.4 Set the columns of $\bm U_n$ equal to the eigenvectors of $\bm \Phi_n^{(1)}$\\
				}
			}
		   5. For each $\bm U_n$ keep the $P_n/2$ eigenvectors with the largest eigenvalues and $P_n/2$ eigenvectors with the smallest eigenvalues.
		\end{algorithm}
		\endgroup}
\end{figure}

\section{Experimental Results}
In this study we  validated the CMP methodology using a widely known and publicly available hyperspectral imaging dataset, named Pavia University, whose number of spectral bands is 103 (see Fig.\ref{fig:datasets}). Ground truth contains 9 classes, while pixels in white color are not annotated.

\begin{figure}[t]
	\centering
	\includegraphics[width=0.52\textwidth]{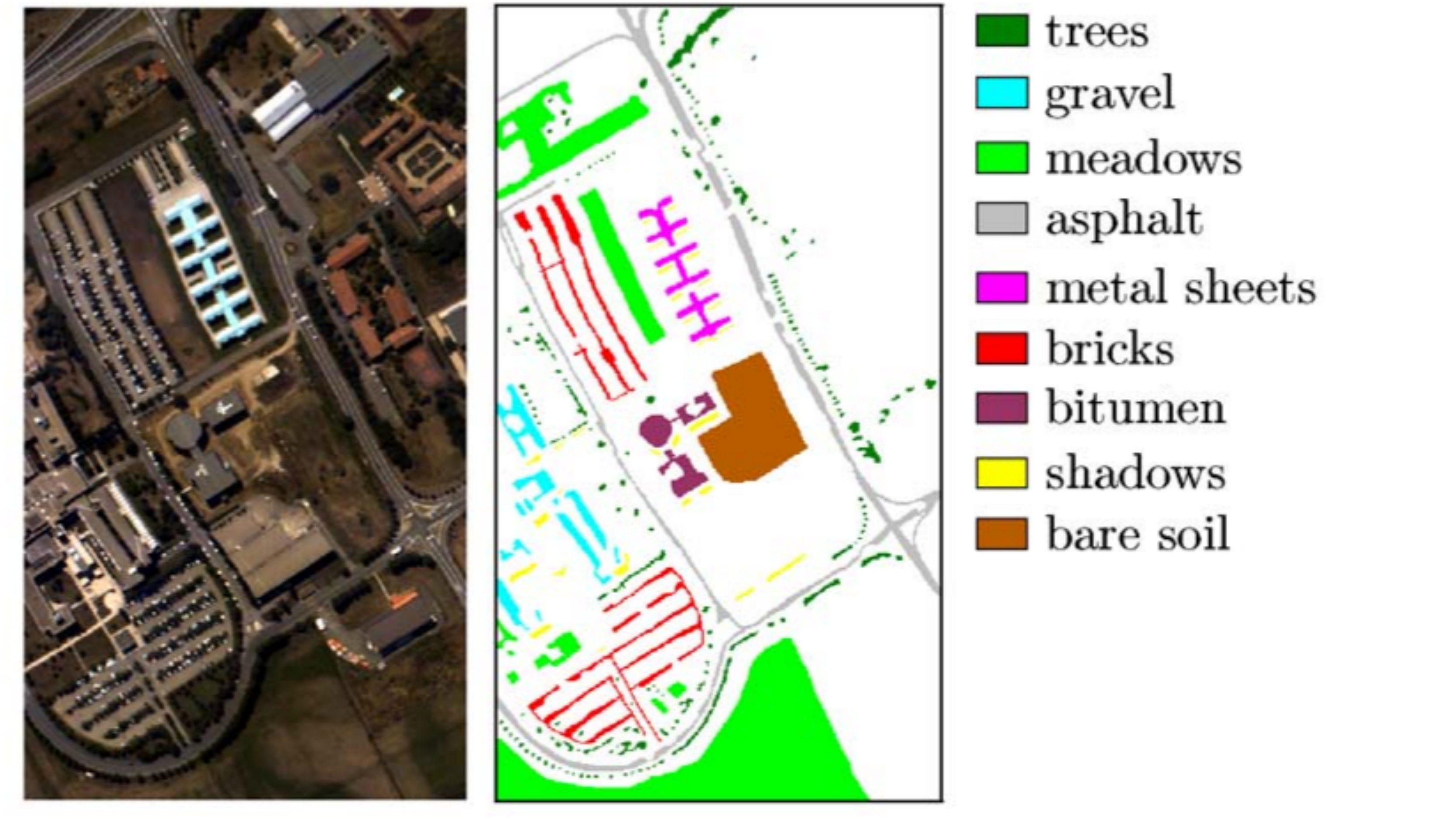}
	\caption{Pavia University dataset (figure taken from \cite{mura2011classification}).}
	\label{fig:datasets}
\end{figure}

The CMP method was developed for binary classification problems. Thus, we grouped together pixels that depict man-made objects and discriminated them from the rest of the pixels. For this dataset, pixels that depict man-made objects are labeled as \textit{asphalt}, \textit{metal sheets}, \textit{bricks} and \textit{bitumen}. Then, the tagged parts of the dataset were split into two sets, i.e. training and testing data. The training set was created by selecting 200 samples from each class.

In order to classify a pixel at location $(x, y)$ on image plane, we followed the approach presented in \cite{makantasis2015deep}, according to which the image is split, along its spatial dimensions, into overlapping patches of size $s \times s \times c$, where $c$ is the number of spectral bands. Then, it is assumed that the label of a pixel located at $(x,y)$ position on image plane, will be the same as the label of the patch centered at $(x,y)$ location.

During experimental validation we compared the proposed CMP method against MPCA \cite{lu2008mpca}, using three different classifiers: Rank-1 Tensor Regression (Rank-1 TR) \cite{zhou2013tensor}, CNN \cite{makantasis2015deep, makantasis2015deepmm}, and Rank-1 FNN \cite{makantasis2017tensor, makantasis2018tensor}. The efficiency of CMP and MPCA was quantified in terms of the classification accuracy of the classifiers on testing set. In our experiments, we set the parameter $s$ equal to 7, and required from both MPCA and CMP to reduce the spatial dimension of the samples to $5 \times 5$ elements. Then, two sets of experiments were conducted. In the first one, the spectral dimensionality of the samples was reduced by selecting: the 26 principal components using MPCA (MPCA-26); the 26 principal components for each pattern class using CMP (CMP-26), and the 13 principal components for each pattern class using CMP (CMP-13), so that the dimensionality along the spectral dimension is 26. In the second one, the spectral dimensionality of the samples was reduced by selecting: the 10 principal components using MPCA (MPCA-10); the 10 principal components for each pattern class using CMP (CMP-10), and the 5 principal components for each pattern class using CMP (CMP-5). In the first experiment the size of the dataset was reduced 4 times, while in the second one 10 times.

\begin{table}[t]
	\centering
	\caption{Overall classification accuracy results (\%).}
	\newcolumntype{L}[1]{>{\hsize=#1\hsize\raggedright\arraybackslash}X}%
	\newcolumntype{C}[1]{>{\hsize=#1\hsize\centering\arraybackslash}X}%
	\label{table:classification}
	
	\begin{tabularx}{0.98\linewidth}{L{4.5}C{5}C{5.2}C{5}}
		\hline \hline 
						& CNN  & Rank-1 FNN & Rank-1 TR \\ \hline
		MPCA-26 & 85.08  & 86.80         & 77.56    \\ \hline
		CMP-26   & \textbf{90.41}  & \textbf{91.25}           & \textbf{77.96}    \\ \hline
		CMP-13   & 88.57  & 88.67          & 76.90    \\ \hline
		MPCA-10 & 83.49  & 84.39          & \textbf{77.59}    \\ \hline
		CMP-10   & \textbf{88.23}  & \textbf{88.31}          & 77.52    \\ \hline
		CMP-5     & 86.76 & 86.27          & 76.08    \\ \hline \hline \\
	\end{tabularx}
\end{table} 

The comparison between MPCA and CMP is presented in Table \ref{table:classification}. The CMP method is more efficient than the MPCA for reducing the dimensionality of tensor objects, regardless of the classification model used, due to the fact that it can exploit labels' information. In other words, CMP is a supervised subspace learning technique, while MPCA is an unsupervised one. For Rank-1 TR the classification accuracy is almost the same both when MPCA and CMP methods are used. This is justified by the fact that Rank-1 TR is a linear classifier and, due its low capacity, cannot perform any better on this dataset.

\section{Conclusion}
In this work, we presented the CMP method, a supervised tensor subspace learning technique, which ensures that tensor objects that belong to different classes will not share common important features after dimensionality reduction. The CMP method was compared against MPCA, and experimental results indicate that it can reduce the dimensionality of tensor objects in a more efficient way. However, the main limitation of CMP is that it is designed for binary classification problems. Therefore, the main focus of our future work is, first, to extend this approach to multi-class classification problems. Another priority of our future work includes the evaluation of CMP efficiency on more datasets with comparisons against other supervised tensor subspace learning methods.

\bibliographystyle{IEEEbib}
\bibliography{refs}

\end{document}